\documentclass[sigconf, nonacm]{aamas}

\usepackage{balance} 

\usepackage{comment,cancel}
\usepackage{algorithm}
\usepackage{algorithmic}
\usepackage{xcolor}
\usepackage{amsmath}
\usepackage{amsthm}
\usepackage[capitalize,noabbrev]{cleveref}
\usepackage{bigints}
\usepackage{multirow}
\usepackage{enumitem}
\usepackage{mathtools}

\usepackage{graphicx}
\usepackage{subcaption}

\newtheorem{theorem}{Theorem}
\newtheorem{definition}{Definition}

\newcommand{\bb}[1]{\mathbb{#1}}

\setcopyright{ifaamas}
\acmConference[AAMAS '25]{Proc.\@ of the 24th International Conference
on Autonomous Agents and Multiagent Systems (AAMAS 2025)}{May 19 -- 23, 2025}
{Detroit, Michigan, USA}{A.~El~Fallah~Seghrouchni, Y.~Vorobeychik, S.~Das, A.~Nowe (eds.)}
\copyrightyear{2025}
\acmYear{2025}
\acmDOI{}
\acmPrice{}
\acmISBN{}

\acmSubmissionID{391}

\title[AAMAS-2025 Formatting Instructions]{Bayesian Collaborative Bandits with Thompson Sampling for Improved Outreach in Maternal Health Program}

\author{Arpan Dasgupta}
\affiliation{
  \institution{Google DeepMind}
  \city{Bengaluru}
  \country{India}}
\email{arpandg@google.com}

\author{Gagan Jain}
\affiliation{
  \institution{Google DeepMind}
  \city{Bengaluru}
  \country{India}}
\email{jaingagan@google.com}

\author{Arun Suggala}
\affiliation{
  \institution{Google DeepMind}
  \city{Bengaluru}
  \country{India}}
\email{arunss@google.com}

\author{Karthikeyan Shanmugam}
\affiliation{
  \institution{Google DeepMind}
  \city{Bengaluru}
  \country{India}}
\email{karthikeyanvs@google.com}

\author{Milind Tambe}
\affiliation{
  \institution{Harvard University, Google DeepMind}
  \city{Boston}
  \country{USA}}
\email{milindtambe@google.com}

\author{Aparna Taneja}
\affiliation{
  \institution{Google DeepMind}
  \city{Bengaluru}
  \country{India}}
\email{aparnataneja@google.com}

\begin{abstract}
    Mobile health (mHealth) programs face a critical challenge in optimizing the timing of automated health information calls to beneficiaries. This challenge has been formulated as a collaborative multi-armed bandit problem, requiring online learning of a low-rank reward matrix. Existing solutions often rely on heuristic combinations of offline matrix completion and exploration strategies. In this work, we propose a principled Bayesian approach using Thompson Sampling for this collaborative bandit problem. Our method leverages prior information through efficient Gibbs sampling for posterior inference over the low-rank matrix factors, enabling faster convergence. We demonstrate significant improvements over state-of-the-art baselines on a real-world dataset from the world's largest maternal mHealth program. Our approach achieves a $16\%$ reduction in the number of calls compared to existing methods and a $47$\% reduction compared to the deployed random policy. This efficiency gain translates to a potential increase in program capacity by $0.5-1.4$ million beneficiaries, granting them access to vital ante-natal and post-natal care information.  Furthermore, we observe a $7\%$ and $29\%$ improvement in beneficiary retention (an extremely hard metric to impact) compared to state-of-the-art and deployed baselines, respectively.  Synthetic simulations further demonstrate the superiority of our approach, particularly in low-data regimes and in effectively utilizing prior information. We also provide a theoretical analysis of our algorithm in a special setting using Eluder dimension.
\end{abstract}


\newcommand{\BibTeX}{\rm B\kern-.05em{\sc i\kern-.025em b}\kern-.08em\TeX}

\begin{document}


\pagestyle{fancy}
\fancyhead{}


\maketitle 


\section{Introduction}
Mobile health (mHealth) programs offer a powerful tool for delivering vital health information, but face a critical challenge: optimizing the timing of automated calls to maximize engagement. This is particularly important in maternal mHealth programs, which play a vital role in reducing maternal mortality rates - a key target within the WHO's Sustainable Development Goals \cite{WHOTargetMaternal}. To ensure these programs achieve their full potential, automated calls must be strategically timed to achieve high pick-up rates, leading to improved health outcomes for mothers and infants.

This paper focuses on Kilkari \cite{kilkari}, the world's largest maternal mHealth program. Implemented nationwide across India by the Ministry of Health and Family Welfare in partnership with the NGO ARMMAN \cite{armmanArmmanHome}, Kilkari delivers critical maternal and child health information through automated voice calls throughout pregnancy and the post child birth period.  Kilkari has served over $40$ million mothers across India so far, with over 3 million active subscribers at any given time.  The importance of listening to these voice messages has been shown to have significant impact on the health outcomes of mothers and babies~\cite{mohan2021can}, particularly among the most marginalised who have the most to benefit from this program, and have the least access to resources. 

However, one major challenge faced by the program is that the pick-up rate of the calls is very low. This is largely due to the fact that different beneficiaries prefer to listen to these calls at different time slots and on different days due to practical constraints such as shared family phones, different working hours, household responsibilities, as well as network reliability, particularly in rural districts of India ~\cite{JJH}.
To address this, the program attempts sending the automated voice calls multiple times in a week, until a call is answered,  
with almost $50\%$ of the economically weakest beneficiaries
requiring more than $6$ attempts on average~\cite{mohan2021can} for a single message in a week and on average
$23\%$ beneficiaries being unreachable despite multiple attempts~\cite{lalan2023analyzing}.
In fact, due to the scale of the program, consistent low listenership of the calls can even lead to beneficiaries being dropped from the program.

Optimizing call timing is therefore essential to improve engagement, and maximize the dissemination of critical health information.  Moreover, it could also help reduce critical bandwidth being spent heavily on retries, which would then enable scaling the outreach of the program to millions of more mothers across the country. While individual preferences would ideally inform call scheduling,  Kilkari's scale prohibits collecting individual time preferences, or demographic information that could help predict those preferences. This necessitates a robust, scalable solution to predict optimal call times based on limited information.


To this end, we formulate the call pick-up problem as a multi-user multi-armed bandit problem, where each time slot represents an arm, and pulling an arm corresponds to sending a call at that time. The reward is based on whether the call is answered. The bandit algorithm must learn which time slot to use for each user to maximize the probability of answering. Current state-of-the-art technique \cite{pal2024improving} formulated this problem as collaborative bandits with a low-rank assumption on user preferences. This low-rank assumption is justified by the observation that groups of users exhibit similar preferences in practice (\cref{fig:pickup_mat} shows an example of the low-rank pick-up problem). This work attempted to collaborate across users to quickly learn their preferences. To this end, the authors employed offline matrix factorization with Boltzmann exploration \cite{cesa2017boltzmann}.  However, this approach is heuristic, lacks theoretical guarantees, and performs sub-optimal exploration, leading to worse regret in practice. Furthermore, it falls short when there are new users joining the program (as in the real world) as it requires a large number of samples to learn the preferences, hence delaying the inference of optimal slots which further increases the risk of drop off from the program. It also lacks the ability to incorporate any available prior information.


In this paper, we propose a Bayesian formulation for the collaborative bandit problem utilizing Thompson Sampling (TS). Bayesian solutions \cite{bharadiya2023review} allow the use of priors to quickly converge to a solution even with limited data, which is crucial in this context, as ineffective exploration could lead to delayed time slot inference and increased drop-off risk.  Thompson Sampling \cite{thompson1933likelihood} is also shown \cite{nakajima2011theoretical} to be a very effective method 
and can provide much tighter bounds than the previously used Boltzmann exploration.


While TS is efficient in terms of regret, its exact implementation is computationally expensive due to the cost of posterior sampling. To address this, we develop a computationally efficient heuristic based on Stochastic Gradient Langevin Dynamics (SGLD) for posterior sampling in TS. Empirically, we show that our algorithm outperforms existing techniques in both cluster and general low-rank settings. We demonstrate its superiority using a real-world dataset from the Kilkari program \cite{ARMMAN_kilkari}, showing how it can significantly increase call engagement and pick-up rates. While we focus on Kilkari, our methods are applicable to a wide range of mHealth programs.

\begin{figure}
    \centering
    \includegraphics[scale=0.35]{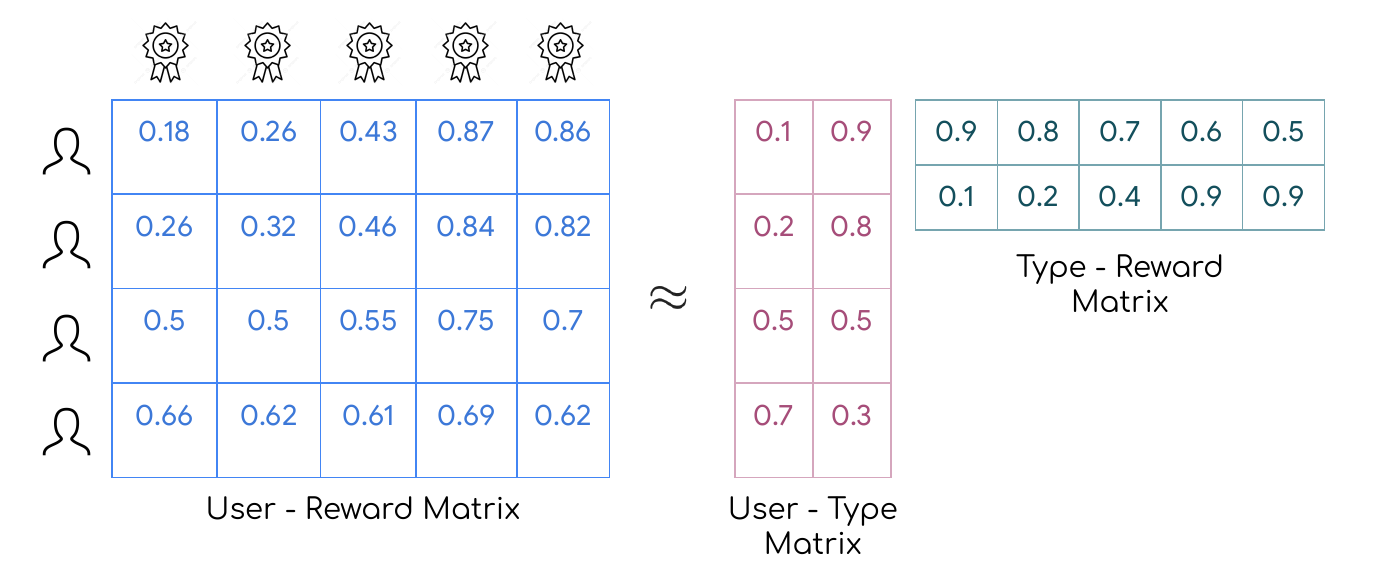}
    \caption{ (Left) Example pick-up matrix-each entry such as 0.18 (top left) represents likelihood of a call being answered by the first beneficiary in the first time slot. (Right) Matrix decomposition into  user X user type and user type X pickup rate probability matrices.}
    \label{fig:pickup_mat}
    \vspace{-0.2in}
\end{figure}

The key contributions of the paper is
a novel Thompson Sampling based algorithm for collaborative multi-armed bandit problem using alternating SGLD for efficient posterior sampling to make the algorithm scalable for potential real world deployment (Section~\ref{sec:proposedmethod}). We also provide the first (to our knowledge) \textit{Eluder dimension} characterization for a special clustered setting of the low rank bandit problem with infinite arms (Section~\ref{sec:bounds_cluster}) which is a subset of the general low-rank case. Eluder dimension is the principal complexity measure that characterizes Bayesian regret for Thompson sampling in general \cite{russo2013eluder}. Finally, we demonstrate empirically the effectiveness of the approach on a time slot inference problem with a simulated and a real world dataset (Section~\ref{sec:experiments}) from the largest maternal mHealth program in the world.

The proposed algorithm exploits the collaborative nature of the problem to infer the optimal time slots to send automated voice messages to beneficiaries. We show that the proposed method reduces the number of call attempts needed to reach out to beneficiaries by a staggering $\textbf{47}$ \textbf{percent} compared to the current deployed system and $\textbf{16}$ \textbf{percent} in comparison to SOTA. Given that the program runs on a scale of millions of beneficiaries nationwide under limited budget constraints, a reduction of $\textbf{16-47}$ \textbf{percent} implies freeing up critical bandwidth \textbf{to potentially enrol $\textbf{0.5-1.4}$ million more mothers}, disseminating critical health information and services to many more mothers from underprivileged communities who may otherwise have limited to no access to resources. We show a further $9$ percent reduction in attempts compared to SOTA needed to reach future new enrolments in the program, leveraging the learnt matrix factorization, freeing up even more bandwidth on a regular basis.  

Finally, due to the scale of the program and logistic constraints, sustained low listenership for several weeks in a row leads to beneficiaries being dropped off from the program. Despite this being an extremely hard metric to move, the proposed method is able to reduce these drop offs by $29$ percent over a period of $4$ months compared to the current deployment and by $7$ percent compared to the SOTA. On the scale of a program operating for $3$ million active subscribers currently, \textbf{an improvement of even $7$ percent translates to $\textbf{210,000}$ mothers being retained in the program} and continuing to benefit from access to valuable health information.



\section{Related Work}

\noindent \textbf{AI in Maternal Healthcare} 
Limited resource allocation problems in maternal healthcare have previously been solved by restless multi-arm bandits \cite{mate2022field, nair2022adviser, verma2023deployed}. The time slot selection problem using collaborative bandits was previously studied by \cite{pal2024improving}. 

\noindent \textbf{Collaborative Multi-armed Bandits}
Multi-armed bandits are a highly studied and effective method for solving several resource allocation problems. Several methods such as phased elimination \cite{lattimore2020bandit, slivkins2019introduction}, UCB \cite{auer2002finite}, Thompson Sampling \cite{thompson1933likelihood, agrawal2012analysis} and Best-arm Identification \cite{agrawal2020optimal, garivier2016optimal} have been studied in detail. The collaborative bandit problem has recently garnered attention due to the widespread popularity of modern recommender systems \cite{bresler2016collaborative, dadkhahi2018alternating}. Under special conditions, several algorithms with strong bounds have been proposed \cite{pal2023optimal, jain2022online}.
An algorithm which is applicable in approximate low rank was proposed by \cite{pal2024improving} which is close to our problem setting. In our work, we empirically compare against this work.


\noindent \textbf{Offline Bayesian Matrix Factorization.}
While several methods have been proposed for Bayesian matrix factorization \cite{nakajima2011theoretical}, the ones which utilize MCMC \cite{salakhutdinov2008bayesian} are of particular interest to us. \cite{ahn2015large} proposes utilizing SGLD in a distributed manner using block partitioning to perform matrix factorization. While such methods have been applied in practice \cite{zhang2020distributed, li2016exploiting}, none of these methods extend the solution in a bandit setting that forms our focus.

\noindent \textbf{Thompson Sampling \& Eluder Dimension.} Thompson Sampling (TS) is a widely used algorithm for bandit optimization, which maintains a prior distribution over unknown problem parameters and updates it sequentially as new data is observed~\cite{thompson1933likelihood}. It recommends actions by sampling from the posterior distribution.  In  a seminal work \citet{agrawal2012analysis} showed that TS obtains optimal regret for Multi-Arm Bandits (MAB). \citet{russo2013eluder} further extended these results, deriving an upper bound for TS's Bayesian-regret in terms of Eluder dimension, for general loss functions. Eluder dimension is a combinatorial quantity which measures the complexity of a function class. However, to the best of our knowledge, Eluder dimension has been characterized for very few function classes such as linear and quadratic~\cite{osband2014model}.


\section{Preliminaries}


\paragraph{Notation.} We use $\mathbb{R}$ to denote the set of real numbers. $\mathbb{R}^{d}, \mathbb{R}^{d_1\times d_2}$ denote the sets of $d$-dimensional vectors, and $d_1\times d_2$ dimensional matrices respectively. For a positive integer $n$, we use $[n]$ to denote the set $\{1, 2 \dots n\}.$ $\Delta_d$ denotes the probability simplex in $\mathbb{R}^d.$  
\paragraph{Problem Formulation.} 
Consider a mHealth program with $N$ beneficiaries (users), $M$ time slots (arms) and $T$ rounds of calling. In the following, we interchangeably refer to the beneficiaries as users and the time slots as arms to describe our techniques in a more general fashion.
Let $\Theta\in \bb{R}^{N\times M}$ be the (user, arm) reward matrix which is unknown to the system. $\Theta_{ij}$ denotes the expected reward of pulling arm $j$ for user $i$. We assume that $\Theta$ can be decomposed as $\Theta \coloneqq UV$, for some latent factors $U \in \mathbb{R}^{N\times C}, V \in \mathbb{R}^{C\times M},$ and $C \leq \min\{M, N\}.$  Special cases of this model include the rank-$1$ setting (where $C=1$), and the cluster setting (where each row of $U$ is a one-hot encoded vector representing user cluster membership). 

In each round $t\in [T]$, a user $u(t)\in [N]$ arrives at the system and is recommended an arm $a(t)\in [M]$. For ease of exposition, we assume the users arrive in a round-robin fashion; that is, $u(t) = t\% N$. At the end of round $t$, the system receives feedback in the form of reward from that user. Let $r_t\in [0,1]$ be the observed noisy reward which satisfies: $\mathbb{E}[r_t] = \Theta_{u(t)a(t)}.$ Note that $a(t)$ can depend on the history $H_t = \{u(s), a(s), r_s\}_{s=1}^{t-1}$.   Finally, we let $\pi$ be the prior distribution over the set of all rank $C$ matrices ($\pi$ encodes the prior knowledge about $\Theta$). The goal of the learner is to optimize Bayesian regret which is defined as
\begin{align}
    R(T; \pi) = \sum_{t=1}^T\mathbb{E}\left[\Theta_{u(t)*} - \Theta_{u(t)a(t)}\right],
\end{align}
where the expectation is taken over the prior distribution of $\Theta$, internal randomness of the system, and $\Theta_{u*}$ denotes the reward of the best item for user $u$.

\paragraph{Thompson Sampling.} As previously mentioned, TS is a widely used algorithm for bandit optimization. At any time step $t$, TS maintains a posterior distribution $\pi_t$ over the set of all rank $C$ matrices, representing its belief about the unknown reward matrix. TS samples a matrix $\hat{\Theta}^{(t)}$ from this distribution and recommends the arm with the highest predicted reward for the current user: $\arg\max_{j\in [M]} \hat{\Theta}^{(t)}_{u(t)j}$. Algorithm~\ref{alg:ts} describes this procedure.
\begin{algorithm}
  \caption{Thompson Sampling}
  \label{alg:ts}
\begin{algorithmic}
  \STATE {\bfseries Input:} number of rounds $T$, prior $\pi$
  \FOR{$t=1$ {\bfseries to} $T$}
  \STATE Sample reward matrix: $\hat{\Theta}^{(t)} \sim \pi_t$
  \STATE Observe user $u(t)$, recommend $a(t) \coloneqq \arg\max_{j\in [M]} \hat{\Theta}^{(t)}_{u(t)j}$
  \STATE Observe reward  $r_t$
  \STATE Update posterior: $\pi_{t+1}(\Theta') \propto \mathbb{P}(\Theta'|H_{t+1})$
  \ENDFOR
\end{algorithmic}
\end{algorithm}

\noindent In the next section, we develop a SGLD based algorithm for efficient posterior sampling in TS and demonstrate its superiority on a real-world maternal mHealth application ( Section~\ref{subsec:realworlddataset}). In Section~\ref{sec:bounds_cluster} we show that TS achieves optimal Bayesian regret guarantees for a cluster setting. 


\section{Proposed Method}
\label{sec:proposedmethod}

\subsection{Stochastic Gradient Langevin Dynamics (SGLD)}

SGLD \cite{welling2011bayesian} is a Markov Chain Monte Carlo (MCMC) method which is used for Bayesian inference and is particularly beneficial in high-dimensional settings. SGLD is able to perform updates in batches, similar to SGD in optimization. The update equation is as follows 
\begin{equation}
    \label{eq:sgld}
    \begin{split}
        \Delta\Theta &= \frac{\epsilon}{2}(\nabla \log{p(\Theta)} 
        + \frac{N}{n}\Sigma_{i=1}^n\nabla \log{p(X_{i} | \Theta})) + \eta \\
        \eta &\sim N(0,\epsilon)
    \end{split}
\end{equation}
where, $\Theta$ are the parameters that characterize the likelihood function of $X$ (observed data). $p(\Theta)$ is the Bayesian prior over these parameters. After an initial set of updates, SGLD generates samples from the posterior distribution of the model parameters, allowing full Bayesian inference instead of just the point estimates.

In this section we go through the steps to apply Thompson Sampling to the low-rank collaborative bandit problem. The basic idea is to utilize SGLD \cite{welling2011bayesian} to be able to perform Bayesian Matrix Factorization as well as posterior sampling.

\subsection{Bayesian Matrix Factorization}
\label{subsec:bmf}

As described in \cref{eq:sgld}, SGLD allows us to compute the updates to the parameters in batches of data. In our case, we assume that the matrices $U, V$ are derived from parameters $u, v$.
Also, in accordance with the assumption, $U_{i,c}$ is obtained as $\frac{e^{u_{i,c}}}{\sum_c' (e^{u_{i,c'}})}$ and each element  $V_{c,j}$ is filled with $\frac{1}{1 + e^{v_{c,j}}}$. The parameters $u_{i,c}, v_{c, j}$ are sampled from the exponential distributions sampled from independent priors each of which is an exponential distribution with pre-decided parameters $\lambda_{i,c}$ and $\alpha_{c,j}$ respectively. Note that the complete set of parameters is $\Theta = \{u_{i,c}\} \cup \{v_{c,j}\}$.

In our setting the prior $p(\Theta)$ from \cref{eq:sgld} is:
\begin{equation}
 \label{eq:prioruv}
   p(\Theta) =  p(u, v) = \left( \Pi_{i,c} \lambda_{i,c} e^{\lambda_{i,c} u_{i,c}} \right) \left( \Pi_{c,j} \alpha_{c, j} e^{\alpha_{c, j} v_{c,j}} \right)
\end{equation}
Coordinates of the gradient $\nabla_{\Theta} \log p(\Theta)$ are given by: 
\begin{equation}
\begin{split}
    \nabla_{u_{i,c}} \log{p(u, v)} = \lambda_{i, c},~
    \nabla_{v_{c, j}} \log{p(u, v)} = \alpha_{c, j}
\end{split}
\end{equation}

Let us assume that the data $X$ is composed of data points $x$ where $x_d = x_{i,j}$ is the reward of the $d_\text{th}$ data point for user $i$ for pulling arm $j$.
The likelihood (second) term of \cref{eq:sgld} is composed of the observed reward $x_{i, j}$ which is a mixture of Bernoulli random variables and is calculated as 
\begin{equation}
    \begin{split}
        Q &= p(x_{i,j} | u, v) = \Sigma_c \left( p(c|u_i) p(x_{i, j}|c, v_{j}) \right) \\
        p(c| u) &= \frac{e^{u_{ic}}}{\Sigma_k e^{u_{ik}}}, ~
        p(x_{i, j}|c, v_{j}) = \frac{x_{i, j}e^{v_{c, j}}}{(1+e^{v_{c,j}})} + \frac{(1-x_{i, j})}{(1+e^{v_{c,j}})}
    \end{split}
\end{equation}
Here, $p(c|u)$ and $p(x_{i, j}|c, v_{j})$ represent the probability of sampling archetype $c$ given the user parameters and the Bernoulli likelihood of observing the Boolean reward $x_{i, j}$ given the rank $c$ and the reward parameters respectively. 
Coordinates of matrix $\nabla_{\Theta} \log p(x_{i, j}|\Theta)$ are hence given by sum :
\begin{equation}
\label{eq:theta_upd}
        \nabla_{u_{i, c}} \log{p(x_{i, j} | u, v)}  = \frac{1}{Q}  (\frac{\delta U_{i,c}}{\delta u_{i,c}}) \cdot p(x_{i,j}|c, v_{j})
\end{equation}
\begin{equation}
\label{eq:r_upd}
    \nabla_{v_{c,j}} \log{p(x_{i,j} | u, v)} = \frac{1}{P} \frac{e^{u{i,c}}}{\Sigma_k e^{u_{i,k}}} \frac{(2x_{i,j} -1) e^{v_{c,j}}}{(1+e^{v_{c,j}})^2}
\end{equation}

\subsection{Thompson Sampling} \label{subsec:ts}

Thompson Sampling \cite{thompson1933likelihood}, \cite{nakajima2011theoretical} is a popular solution paradigm for multi-armed bandits as it provides strong theoretical guarantees. TS requires maintaining a belief, acting according in the best possible way according to the belief and updating it as new evidence is collected. The most difficult step involved is the sampling step and it can be done effectively by MCMC methods like SGLD as they converge to the posterior distribution and generate samples from it. In our problem, we perform the updates to the parameters $u,v$ and calculate $P = UV$ to be the final matrix, from which we choose the arm with the highest reward for each user.

\subsection{Scaling the Algorithm}

While SGLD is able to significantly speed up the posterior sampling process, the algorithm is still slow. In order to speed this up, we come up with a way to scale it in distributed settings. We observe that the updates to $u$ are independent across users given $v$, ie. updates of one user $u_i$ are independent from the updates of $u_j$ assuming $v$ to be fixed.


\begin{theorem}
\label{thm:condind}
    The updates in parameters for one user are independent from the other users.
    \begin{equation}
        P(u_k| v, u_{l\ne k}, X) = P(u_k| v, X)
    \end{equation}
\end{theorem}

\begin{proof}
Given data $X$, parameters $u,v$, the likelihood function is:
\begin{equation}
    L(X, u, v) = \prod_{k=1}^{|X|} \Bigl((1-X_{k(i,j)})(1-u_i v_j)
    + (X_{k(i,j)} u_i v_j) \Bigr)
\end{equation}
where $u_i$ is the $i_\text{th}$ row of $u$ and $v_j$ is the $v_\text{th}$ column of $v$.

We want to establish that the updates of the conditional posterior of the parameters of the $k_\text{th}$ user is independent of the parameters of the other users.
\[
    P(u_k| v, u_{l\ne k}, X) = \frac{L(u_k, r, \theta_{l\ne k}, X) p(u_k)}{\bigintss L(u_k, v, u_{l\ne k}, X) p(u_i) du} 
\]

We can separate out the items in $X$ where user $k$ is involved. 
\begin{align}
    L(&u_i, v, u_{l\ne k}, X) = L(u, v, X) \\
    =& \prod_{l=1}^{N} \left( (1-X_{l(i,j)})(1-u_i v_j) + (X_{l(i,j)} u_i v_j) \right) \\
    =& \left( (1-X_{k})(1-u_i v_j) + (X_{k} \theta_i v_j) \right) \\
     &\qquad  \prod \left( (1-X_{l\ne k})(1-u_{i'} v_{j'}) + (X_{l\ne k} u_{i'} v_{j'}) \right)
\end{align}

The second term in the product comes out in both numerator and denominator (as it is not dependent on $u_k$) and cancels out. Thus we are left with
\begin{align}
    P(u_k| v, u_{l\ne k}, X) &= \frac{L(u_k, v, X_k) p(u_k)}{\bigintss L(u_k, v, X_k) p(u_k) du}
    &= P(u_k| v, X_k)
\end{align}
Where $X_k$ is the data where $k$ is involved.
\end{proof}

\cref{alg:sgld_full} and \cref{alg:sgld_alternate} explain the full and alternating sampling SGLD methods. The complete algorithm is described by \cref{alg:proposed}.

\begin{algorithm}
  \caption{SGLD with Full Sampling}
  \label{alg:sgld_full}
\begin{algorithmic}
  \STATE {\bfseries Input:} Batch Size $n$, Data $X$
  \STATE {\bfseries Hyperparameters:} Learning Rate $\epsilon$
  \STATE {\bfseries Parameters:} $\lambda, \alpha$
  \STATE Initialize $u$ and $v$ 
  \REPEAT
  \STATE Select batch $x$ of size $n$ from $X$.
  \STATE Calculate terms from \cref{eq:theta_upd} and \cref{eq:r_upd}
  \STATE Update $u, v$ using \cref{eq:sgld}.
  \UNTIL{$u, v$ converge}
  \STATE {\bfseries Return} {$u, v$}
\end{algorithmic}
\end{algorithm}

\begin{algorithm}
  \caption{SGLD with Alternating Sampling}
  \label{alg:sgld_alternate}
\begin{algorithmic}
  \STATE {\bfseries Input:} Batch Size $n$, Data $X$, User Blocks $b$
  \STATE {\bfseries Hyperparameters:} Learning Rate $\epsilon$
  \STATE {\bfseries Parameters:} $\lambda, \alpha$
  \STATE Initialize $u$ and $v$ 
  \REPEAT
  \STATE Select batch $x$ of size $n$ from $X$.
  \FOR{$b_i$ {\bfseries in} b}
  \STATE Calculate terms from \cref{eq:theta_upd} for $b_i$
  \STATE Update $u_{b_i}$ using \cref{eq:sgld}.
  \ENDFOR
  \STATE Merge $u_{b_i}$ to get $u$.
  \STATE Calculate terms from \cref{eq:r_upd}
  \STATE Update $v$ using \cref{eq:sgld}.
  \UNTIL{$u, v$ converge}
  \STATE {\bfseries Return} {$u, v$}
\end{algorithmic}
\end{algorithm}

\begin{algorithm}
  \caption{ Proposed Algorithm \textbf{TS-SGLD}}
  \label{alg:proposed}
\begin{algorithmic}
  \STATE {\bfseries Input:} Time steps $T$, Samples per time step $s$.
  \STATE Generate data $X_0$ from random users and time slots.
  \FOR{$i=1$ {\bfseries to} $T$}
  \IF{Sampling Method == \textit{full}}
  \STATE $u, v$ from \cref{alg:sgld_full}
  \ELSIF{Sampling Method = \textit{alternating}}
  \STATE $u, v$ from \cref{alg:sgld_alternate}
  \ENDIF
  \STATE Calculate $U, V$ 
  \STATE Generate $X_i$ with $s$ samples using TS. 
  \ENDFOR
  \STATE {\bfseries Return} $U$, $V$
\end{algorithmic}
\end{algorithm}

\cref{thm:condind} allows us to compute updates in batches of users assuming $v$ to be constant. The batch updates to $u$ can be accumulated and then the updates to $v$ can be computed. The updates to $v$ also follow the conditional independence property, but the limiting factor in a majority of the cases is the number of users in $u$. These alternating updates are easier to perform than those in \cite{ahn2015large} as the recombination step is much easier.


\section{Theoretical Bounds - Clustered Low Rank Bandits}
\label{sec:bounds_cluster}

The general problem of low rank bandits in the Bayesian setup can be specified as follows. There is a space of low rank reward matrices that parameterized the problem given by: $\Omega_C \subset \mathbb{R}^{N \times D}$  such that $ \mathrm{rank}(R) = C, \forall R \in \Omega_C $. Here, we have $N$ users and $D$ is the dimension of the action space. Let the unknown reward matrix be $R$. We have a prior $\pi(R)$ over the space $\Omega$.

\textbf{Reward Likelihood Model:} For user $u$ and action vector $a \in \mathcal{A} \subset \mathbb{R}^D$, the reward model given $R$ is given by:
\begin{align} \label{reward_lkhood}
    Y(u,a) \sim R(u,:)^T a + N 
\end{align}
where $N$ is i.i.d $1$-sub-Gaussian random variable.  

\textbf{Bounded Parameter Space:} We will assume that $\lVert R(u,:)\rVert_2 < S, \forall R \in \Omega $, $\lVert a \rVert_2 < \gamma, \forall a \in \mathcal{A}$.

\textbf{Bayesian Regret:} Let user $u_t$ arrive at time $t$ and $a_t$ be the action chosen according to some bandit algorithm given previous actions $(u_1,a_1), \cdot \cdot (u_{t-1},a_{t-1})$ and the corresponding rewards $y_1 \ldots y_{t-1}$. We are interested in Bayesian regret given by:
\begin{align}
\mathbb{E}_{R \sim \pi, y(u_t,a_t) \sim Y(u_t,a_t)|R,u_t, \sim \mathrm{Unif}[1:N]}[ R(u_t,:)^T a_t - R(u_t,:)^T a^*(u_t)] \nonumber
\end{align}

Here, $a^*(u_t) = \arg \max \limits_{a \in {\mathcal A}} R(u_t,:)^T a$

We analyze the Bayesian Regret for the natural Thompson Sampling algorithm with perfect Gibbs Sampling from the posterior at every step.

\textbf{Thompson Sampling:} At every time $t$, let the history of users, actions and rewards be denoted $H_t$  given user $u_t$, compute posterior $\pi(R|H_t)$. Sample $\hat{R} \sim \pi(R|H_t)$. Play action $a_t= \arg \max \limits_{a \in \mathcal{A}}\hat{R}(u_t,:)^T a$.

The basis for our regret bounds is the seminal result of \citet{russo2013eluder} which upper bounds the Bayesian regret of TS in terms of Eluder dimension, a measure of the complexity of a model class.  Specifically, for a model class $\Omega_{C} = \{UV:  U \in \mathbb{R}^{N\times C}, V \in \mathbb{R}^{C\times D}\}$, the Eluder dimension is defined as the maximum length of a sequence of actions that consistently provide new information about the underlying function. Formally, an action $u_t,a_t$ is said to be $\epsilon$-dependent on previous action sequence $\{(u_1,a_1), \ldots (u_{t-1},a_{t-1})\}$ if any pair of matrices $R, R'\in \Omega_{C}$ satisfying $\sqrt{\sum_{s=1}^{t-1}((R_{u(s),:} -R'_{u(s),:})^T a_s) ^2} \leq \epsilon$ also satisfies: $|( R'_{u(t),:} - R_{u(t),:})^Ta_t|\leq \epsilon$. $u_t,a_t$ is $\epsilon$-independent of $\{(u_1,a_1), \ldots (u_{t-1},a_{t-1})\}$ if it is not $\epsilon$-dependent on it. 
\begin{definition}[\citet{russo2013eluder}] The $\epsilon$-Eluder dimension $\text{dim}_E(\Omega_{C}, \epsilon)$ is the length of longest sequence of elements in $[C]\times [D]$ such that for some $\epsilon' \geq \epsilon$, every element is $\epsilon'$-independent of its predecessors.
\end{definition}
\noindent\citet{russo2013eluder} show that the Bayesian regret of TS is upper bounded (upto logarithmic factors) by
\[
\sqrt{\text{dim}_E(\Omega_{C}, T^{-2}) \log N(\Omega_{C}, T^{-2}, \|\cdot\|_{\infty})T},
\]
where $N(\Omega_{C}, \epsilon, \|\cdot\|_{\infty})$ is the covering number of the model class. We see that $\epsilon$-Eluder dimension is the key complexity measure that governs regret in bandit problems. Our main contribution is characterization of $\epsilon$-Eluder dimension for the cluster model below which is a special case of low rank bandits with infinite arms. We describe the simpler model below.

\textbf{Cluster Model with infinite arms:} We further consider $\Omega_C$ with a clustering assumption that is there are only $C$ distinct rows in $R $ and every user $u$ is mapped to a cluster using a cluster assignment function $c:[N] \rightarrow [C]$. Therefore, one can reparameterize  $\Omega_C = (c,R),~c:[N] \rightarrow [C],~ R \in \mathbb{R}^{C \times D}$. $R(u,:) = R(c(u)
,:)$ for the reward model \eqref{reward_lkhood}.

We will show that the $\epsilon-$ Eluder dimension for the above model is :
$\mathrm{dim}_E(\Omega_C, \epsilon) = \mathcal{O}(2CD \log (1+2S/\epsilon^2) + CN)$. This is our key theoretical contribution.

While our main contribution is the $\epsilon$- Eluder dimension characterization of the space of reward matrices for cluster model with infinite arms, we first establish $\epsilon$- Eluder dimension of the finite arms case before proceeding to the general case as this is very insightful and simpler to understand.

\textbf{Cluster Model with finite arms:} We consider a simpler cluster problem where the action set ${\mathcal A}=\{e_1, e_2 \ldots r_D\}$ consists of indicator vectors. We denote this as the cluster model with finite arms.

\begin{theorem} In the cluster model, for $C$ user-clusters with a total of $N$ users and $D$ finite arms $D$ arms, the $\epsilon$-Eluder dimension is at most $(2D+N)C$.
\end{theorem}
\begin{proof}[Proof Sketch]
 We first take a specific realization of $(c,R)$ and with respect to a sequence of actions look at the reward sequence which is $\epsilon$ Eluder. For every new action, some subset of alternate possible reward matrices get separated from this by taking a value $\epsilon$ far until only one is left. We argue that while separation happens you either learn something about the assignment function $c$ for some $c(u)$ or learn something about entry in $R$. We then argue that corresponding Eluder subsequences where one or the other is learnt has lengths at most $CN$ and $CD$ respectively. Detailed formal argument is presented in the supplement.
\end{proof}

\begin{theorem}\label{infinite_thm} In the cluster model, for $C$ user-clusters with a total of $N$ users and $D$-dimensional infinite arm set $\mathcal{A}$, the $\epsilon$-Eluder dimension is $O(2CD \log (1+2S/\epsilon^2) + CN))$.
\end{theorem}
\begin{proof}[Proof Sketch]
It follows the finite arms case except that we use a reduction to linear bandits to argue that the Eluder subsequence where new information about $R$ is learnt cannot exceed $\tilde{O}(CD \log (1/\epsilon))$. Exact formal argument in given in the supplement.
\end{proof}

\begin{theorem}
For the cluster model in the infinite arm case, we have:

a) the log-covering number $\log N(\Omega_C, \epsilon, ||.||_\infty) = O(CD \log(1/\epsilon)+ N \log C)$ and,

b) $\mathrm{dim}_E(\Omega_C, \epsilon) = \mathcal{O}(2CD \log (1+2S/\epsilon^2) + CN)$ (Theorem \ref{infinite_thm})

Thus propositions 4 and 5 from \cite{russo2013eluder} with $\epsilon = T^{-2}$ yields a Bayesian regret bound of:

\[\tilde{O}\left(CD\sqrt{T\left(1+\frac{N}{D}\right)\left(1+\log(C)\frac{N}{CD}\right)}\log(T)\right)\] .
\end{theorem}

\begin{proof}
 Theorem \ref{infinite_thm} shows that $\mathrm{dim}_E(\Omega_C, \epsilon) = \mathcal{O}(2CD \log (1+2S/\epsilon^2) + CN)$. In the clustered bandits case with infinite arms, the reward matrix has two components, $R \in \mathbb{R}^{C \times D}$ and an assignment function $c:[N] \rightarrow C$. If we want to cover space of matrix $R$ in infinity norm, then we would exactly cover all assignment functions and there are $C^N$ in number and for every $R$ we would find an approximation which is $\epsilon$ far in the infinity norm. The latter $\epsilon$ cover would involve $O(\frac{1}{\epsilon})^{CD}$ matrices assuming a bounded space of $C \times D$ matrices. Therefore, the covering number is $O(\frac{1}{\epsilon})^{CD} * C^N$. This yields the log-covering number bound.Combining both by results in \cite{russo2013eluder} gives the regret bound.
 \end{proof}
\textbf{Remark:} Assuming every user arrival (random uniform arrival) is treated as an independent linear bandit, regret would be roughly $\tilde{O}(D\sqrt{NT})$. When $N \approx D$ (i.e. $N$ and $D$ are comparable), our regret bound is effectively $ \tilde{O}(CD \sqrt{T})$ ignoring $\log$  and constant factors where only the number of clusters and $D$ matters. Usually $C$ is a constant while $N$ and $D$ are very high dimensional. In this case, our regret bound gives a polynomial advantage over treating each user as an independent bandit. When $N>>D$, we point out that we have an extra factor $\sqrt{(1+ \log C \frac{N}{CD})}$ in regret that appears primarily because of the log covering number dependence in regret. We believe that a tighter analysis connecting regret and Eluder dimension is needed in this specific clustering setting to resolve this.


\section{Experiments}
\label{sec:experiments}
\subsection{Baselines}

We show the efficacy of the proposed algorithm \textbf{TS-SGLD} (Algorithm~\ref{alg:proposed}) by comparing against other methods performing low-rank collaborative bandits. 
We compare our method against the SOTA Phased Matrix Completion (Phased MC) and Greedy Matrix Completion (Greedy MC) from \cite{pal2024improving}. Other baselines include Alternating Linear Bandits (AltMin) from \cite{dadkhahi2018alternating} and LATTICE from \cite{pal2023optimal}, where the latter is specifically for comparison on the simulated cluster setting. We also consider UCB on individual arms and a random baseline which is currently deployed by the NGO for the real world maternal health application. Results are shown for the best finetuned hyperparameters for all baselines to give a fair comparison.

\subsection{Simulated Data}

The simulated dataset is created by generating random matrices of size $N \times M$ with rank $C$ by multiplying two random matrices $U$ and $V$ of shape $N \times C$ and $C \times M$ respectively and adding Gaussian noise to it. This is followed by a normalization step which scales the values to probabilities between $[0,1]$.
$$
P = \textit{normalize}(U \cdot V + N(\mu, \sigma^2))
$$
The matrices $U,V$ are randomly generated with each element being from the range $[0, 1]$ and $\mu$ and $\sigma$ are chosen to be $0.5$ and $0.1$ respectively. The \textit{normalize} operation maps the smallest and highest values to $0$ and $1$ and scales the other values accordingly. $C$ is set to $4$ for these experiments. Uniform prior of 0.5 is applied to all parameters in Equation~\ref{eq:prioruv}.

\begin{figure}
    \begin{subfigure}{.35\textwidth}
        \begin{center}
        \centerline{\includegraphics[width=\columnwidth]{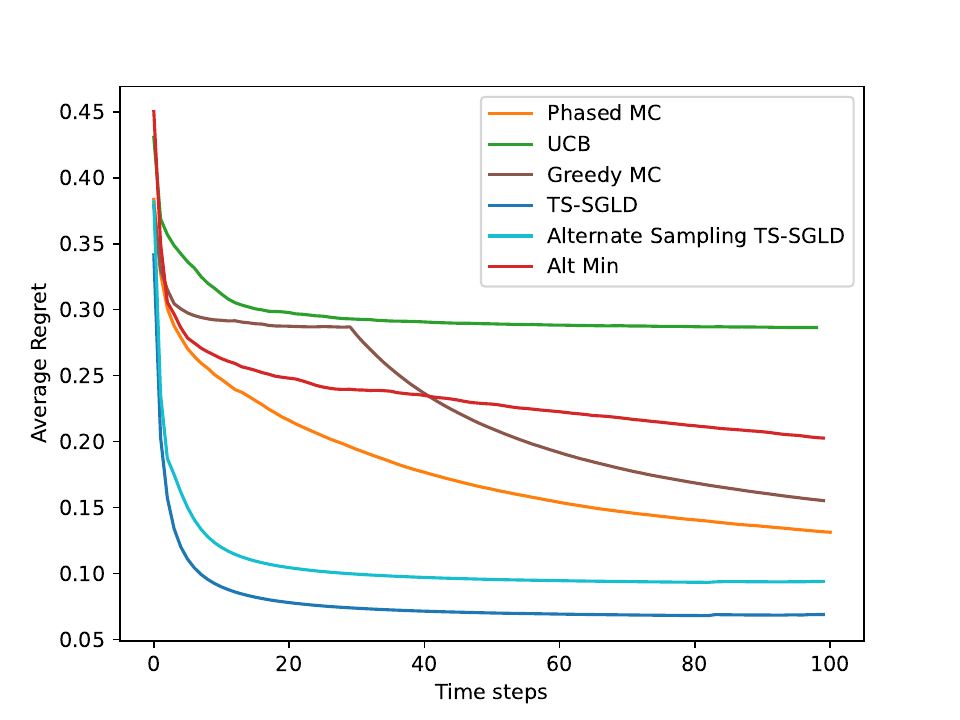}}
        \end{center}
    \end{subfigure}
    \hspace{3mm}
    \vspace{-0.15in}
    \caption{Regret for the low-rank case on simulated data: Average regret for different methods averaged over $15$ random matrices. Results are on a $1000$ users and $20$ arms matrix with $4$ user types. Every time step adds $1000$ samples.}
    \label{fig:regret_low_rank_simulated}
    \vspace{-0.15in}
\end{figure}

 \cref{fig:regret_low_rank_simulated} compares regrets of different methods for $1000$ users and $20$ arms over $15$ randomly generated low-rank matrices.
 UCB gives the highest regret since it cannot leverage the collaborative structure of the arm rewards. Due to the low-data setting, \textbf{Greedy MC} takes more exploration steps to catch up with \textbf{PhasedMC}. \textbf{GreedyMC} is primarily an explore-then-exploit approach which is prohibitive since sub-optimal choices at the start can be costly, especially in the subsequent real-world application.
\textbf{TS-SGLD} with full sampling performs about $\textbf{~65\%}$ better than \textbf{PhasedMC}. \textbf{TS-SGLD} with alternate sampling has higher regret than \textbf{TS-SGLD} with full sampling since alternating sampling leads to a noisier convergence to the posterior, but still has lower regret compared to all other baselines. Similar results are also shown for the special cluster setting described in  Section~\ref{sec:bounds_cluster} in \cref{fig:regret_cluster_simulated} in the Appendix. LATTICE \cite{pal2023optimal} is known to provide optimal bounds for the cluster case in theory. However in practicality, in the low data regimes where there is a small amount of data coming in at each time step, the regret of LATTICE is worse than \textbf{TS-SGLD} with full sampling by $15$ percent in final cumulative regret. 

\subsection{Real World Dataset from Maternal mHealth program}
\label{subsec:realworlddataset}

\textbf{Dataset}: The problem of optimal time slot selection for beneficiaries in a mHealth program can be formulated as a collaborative bandits problem \cite{pal2024improving}. The dataset is created from historic data from the Kilkari program where anonymized call logs provide information about which user was called at which time slot, and whether the call was answered to provide a groundtruth pickup matrix (consisting of $\textbf{200,000}$ users observed over one year for $7$ time slots daily between $8am$ and $8pm$) and 
a test dataset corresponding to a different set of $\textbf{1000}$ users. For the experiments, timeslots are discretized into $7$ and $14$ slots 
where $14$ slots correspond to different slots for weekdays and weekends. $14$ slots increases the problem complexity, however it is more representative of this domain given that many users have limited access to shared phones on weekdays and may have different preferences for weekends.

\textbf{Prior}: Prior information is obtained from previous iterations of the calling program. 
It was observed that the real world dataset for a significant number of users contains zero entries in the pickup matrix i.e. zero pick-up rates across most time slots. This was provided as a prior on the $V$ matrix on the second term in \cref{eq:prioruv}. No priors were provided on the $U$ matrix, since there is no demographic or any other meaningful information available about the users to provide meaningful priors on the user type categorization matrix. However, the method allows for priors to be incorporated on both $U$ and $V$ matrices if such information were available.

The method requires that we choose an appropriate rank $C$ which fits the data reasonably. In all of our experiments for this dataset, $C=5$ was chosen, based on experiments shown in Figure ~\ref{fig:regret_rank} in the Supplementary material.
Users are split into different buckets based on largest pick-up probability across time slots to analyze the performance of the algorithms finely. Beneficiaries with highest pick-up rate (across any time slot) below $0.2$ and above $0.8$ are classified into the low and high pick-up rate groups respectively. The remaining users are classified into the middle pick-up rate group. On a dataset of $1000$ users, $488$ users fall in this category.

\textbf{Impact on Regret}: \cref{fig:regret_middle_cum} compares average regret on the medium listenership bucket of users across different baselines, and and \cref{fig:regret_1} (in supplementary) shows regret across all users. \textbf{TS-SGLD} performs particularly well in low-data settings such as these. \textbf{TS-SGLD} with prior, is able to effectively utilize prior information and outperforms \textbf{TS-SGLD} without prior by $8$ percent as well as \textbf{PhasedMC} and \textbf{AltMin} by $14$ percent and $21$ percent respectively. Despite the extremely low pickup and high pickup users being harder to impact, \textbf{TS-SGLD} improves these by $4$ and $8$ percent respectively compared to SOTA.

\textbf{Impact on re-attempts}:
Since a lot of bandwidth is currently spent on re-attempts to increase likelihood of pickup across users, we also measure the expected number of attempts required by the different methods. The program is known to implement a constraint of $\le 9$ calls \cite{mohan2021can}
The average number of attempts made for connecting a call is ~$5$ \cite{mohan2021can}. \cref{fig:attempts_middle} shows the average number of attempts made for the beneficiaries in the middle pick-up rate group (~$49\%$ of beneficiaries) reduces by $16\%$ compared to the SOTA and $46\%$ compared to the current deployed system.

\textbf{Weekday vs Weekend slots}: Figure~\ref{fig:attempts_weekday} shows the number of attempts reduces significantly compared to the SOTA and non-collaborative baselines, when $14$ slots are used instead of $7$ slots to take into account additional preference for weekday vs weekends.

\textbf{Impact on new enrollments}:
Since the program enrolls new beneficiaries at regular intervals,  \cref{fig:regret_new} and \cref{fig:attempts_new} show the performance of the algorithm for new users after the algorithm has been running on an existing cohort for $10$ timesteps for both the algorithms.
The regret achieved is significantly lower for \textbf{TS-SGLD} compared to \textbf{PhasedMC} by $18$ percent and translates to a further $9$\% reduction in the number of attempts to connect calls to the new users.
In fact, due to the $UV$ factorization, the algorithm only has to learn the $U$ matrix entries for the new users, and can leverage the already known $V$ matrix to infer slots sooner. Since these mHealth programs suffer from a high rate of drop off \cite{ARMMAN_kilkari} due to prolonged low listenership, it is important not only to learn optimal time slots accurately but also quickly with a few number of samples to avoid risk of drop off from the program.

\textbf{Impact on Drop offs}:
Currently, beneficiaries with low listenership (< $25\%$ of message length) for 6 consecutive weeks or low listenership for 9 weeks within a 16 week period are dropped off from the program. It is therefore imperative to send calls at a convenient time to both boost listenership and retain beneficiaries in the program to ensure continued access to critical health information. Despite this being a very strict metric and considerably harder to impact, \textbf{TS-SGLD} manages to reduce dropoffs by $7$ percent compared to the SOTA and by $29$ percent compared to the current deployed system (Figure~\ref{fig:dropoffs}).

\textbf{Combined Pickup and Engagement data}:
Similar to \cite{pal2024improving}, pick-up and engagement data was combined to minimize the regret for engagement. 
A call is said to be \textit{engaged} if listenership is greater than $25\%$ of the message length. \textbf{TS-SGLD} on concatenated pickup and engagement data shows $17$ percent improvement (Figure \ref{fig:attempts_pe}) in re-attempts over SOTA on medium-bucket listenership users.

\textbf{Simulated Data vs Real World Data Gap}:
The difference in performance of the algorithm
between the simulated and real-world data is due to the increased amount of unstructured noise in real datasets. Furthermore, the low-rank assumption in the real-world dataset does not hold completely, and in fact the matrix is approximately low-rank, as discussed in Supplementary Section~\ref{sec:properties}.


\begin{figure*}
    \begin{subfigure}{.33\textwidth}
        \begin{center}
        \centerline{\includegraphics[width=\columnwidth]{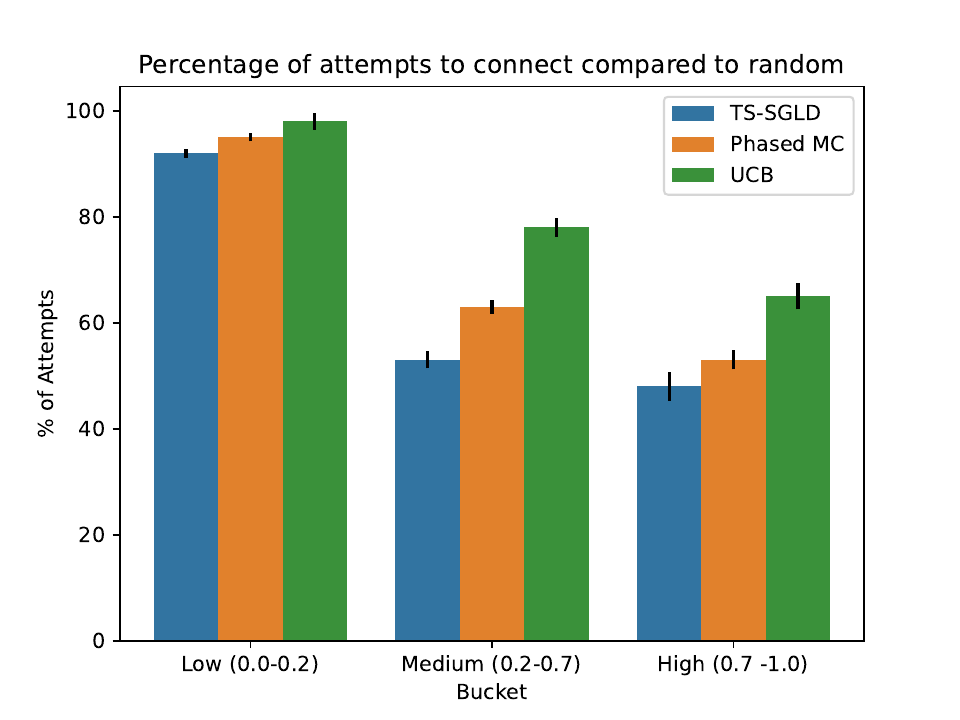}}
        \caption{Number of attempts for connecting $1000$ calls for $1000$ users using pickup data}
        \label{fig:attempts_middle}
        \end{center}
    \end{subfigure}
    \hspace{3mm}
    \begin{subfigure}{.33\textwidth}
        \begin{center}
        \centerline{\includegraphics[width=\columnwidth]{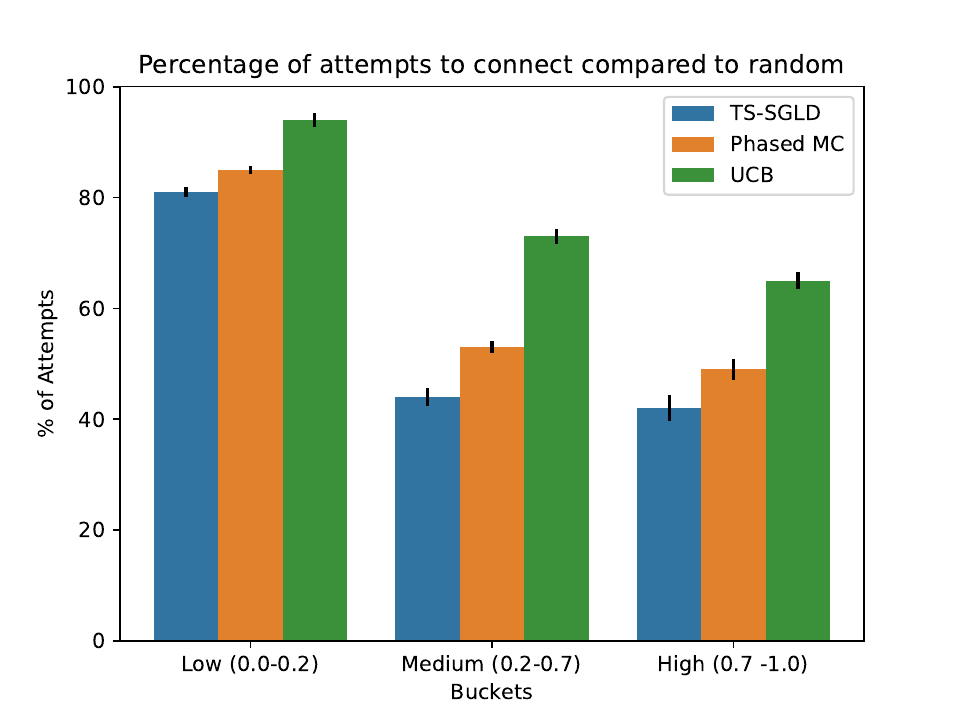}}
        \caption{Number of attempts for connecting $1000$ calls for $1000$ users using pickup + engagement data}
        \label{fig:attempts_pe}
        \end{center}
    \end{subfigure}
    \begin{subfigure}{.33\textwidth}
        \begin{center}
        \centerline{\includegraphics[width=\columnwidth]{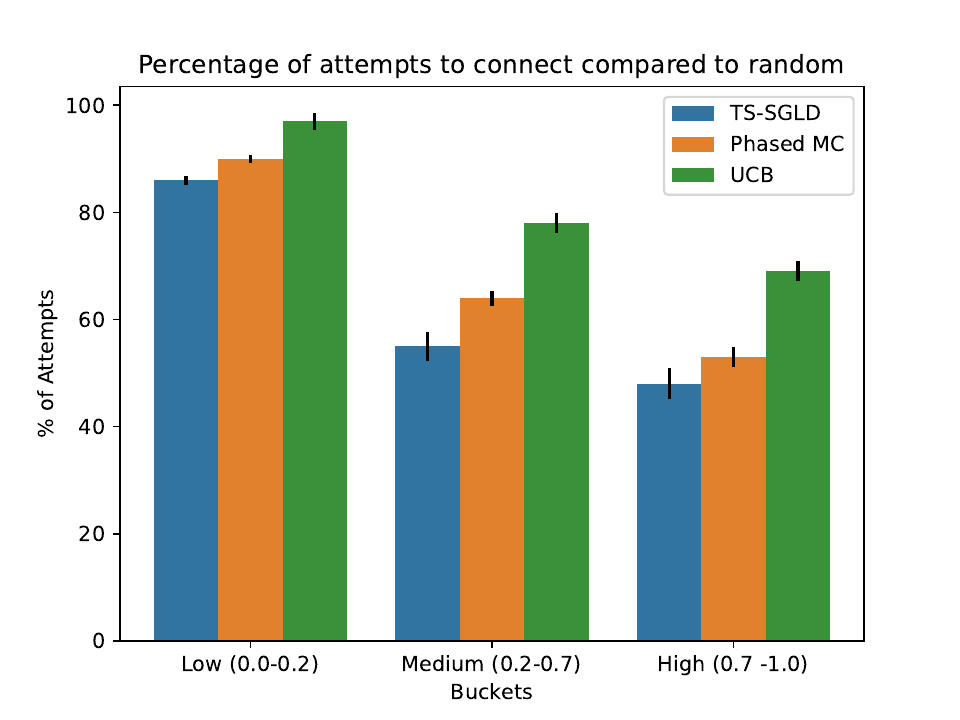}}
        \caption{Number of attempts connecting $1000$ calls for $1000$ users using weekend + weekday slots}
        \label{fig:attempts_weekday}
        \end{center}
    \end{subfigure}
    \hspace{3mm}
    \begin{subfigure}{.33\textwidth}
        \begin{center}
        \centerline{\includegraphics[width=\columnwidth]{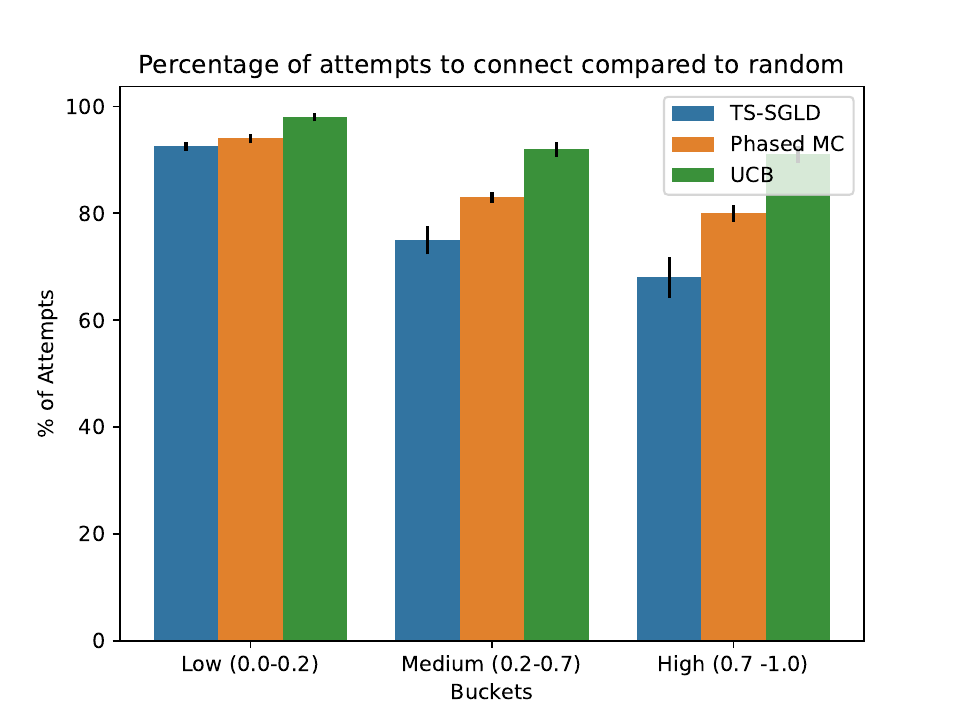}}
        \caption{Number of attempts connecting $1000$ calls for $100$ new users added to the system}
        \label{fig:attempts_new}
        \end{center}
    \end{subfigure}
    \vspace{-0.15in}
    \caption{Number of attempts needed to reach out to beneficiaries in the real-world ARMMAN dataset across $3$ listenership buckets. All plots are relative to the deployed random baseline capped at $100\%$.}
    \vspace{-0.1in}
\end{figure*}


\begin{figure}
    \begin{subfigure}{.5\textwidth}
        \begin{center}
        \centerline{\includegraphics[width=\columnwidth]{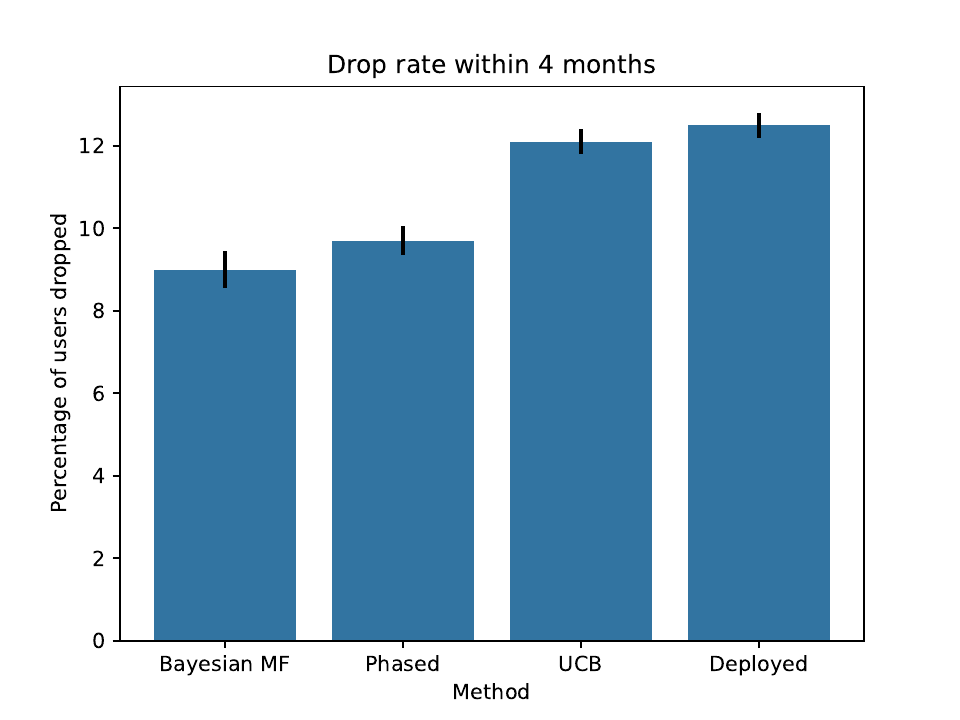}}
        \end{center}
    \end{subfigure}
    \caption{Percentage of dropoffs over a $4$ month period. Dropoffs happen when the engagement goes below $25\%$ for $6$ weeks consecutively or $9$ weeks in a $12$ week period.
    }
    \label{fig:dropoffs}
\end{figure}

\section{Conclusion and Future Work}

We presented an algorithm to solve a collaborative bandits problems for low rank matrices using Thompson sampling. The algorithm was run on time slot inference problem inspired by the real world dataset obtained 
from the largest maternal mHealth program in the world.
 The proposed method showed a significant reduction in the number of call attempts needed to reach out to beneficiaries by $47$ percent compared to the current deployment and $16$ percent compared to the SOTA matrix completion methods, and a further $9$ percent for new enrolments in the program leveraging the already inferred user type to slot preference information from matrix factorization. Reduced attempts free up critical bandwidth, to enable the program operating under budget restrictions to potentially enroll $0.5$ to $1$ million more mothers (assuming on average 5 attempts per voice message) into the program. The method also led to a reduction in drop offs by $7$ and $29$ percent compared to the SOTA and the current deployment, which can effectively enable retaining $0.2$ to $1$ million mothers into the program, ensuring continued access to critical health information for marginalised communities with limited access to resources. We further strengthened our approach by utilizing both the pickup and engagement data.
  To the best of our knowledge, we are also the first to provide Eluder dimension based analysis for Thompson Sampling for the cluster case which is a subset of the general low rank case. While the proposed method can incorporate priors on both user type and user preference matrices, we currently incorporate priors learnt from previous calling data only on the latter due to lack of demographic information about beneficiaries. 

\section{Ethical Considerations}

Acknowledging the responsibility associated with real-world AI systems for undeserved communities, we have closely coordinated with
domain experts from the NGO throughout our analysis. This study
falls into the category of secondary analysis of the aforementioned
dataset. We use the previously collected engagement trajectories of
different beneficiaries participating in the service call program to
train the predictive model and evaluate the performance. All the data collected through the program is owned by the NGO and only the
NGO is allowed to share data while the research group only accesses
an anonymized version of the data.

\textbf{Bias and fairness} There is no demographic data available
for Kilkari beneficiaries. Nonetheless, prior studies such as ~\cite{mohan2021can} point out that exposure to Kilkari helps
improve health behaviors among the most marginalised, thus
helping to close the gap across some inequities in the population. They also indicate that the more marginalised population benefits from higher number of retries in Kilkari calls.
The proposed methodology potentially helps reduce inequities by reducing
the number of retries required, and by improving listening
of Kilkari messages particularly amongst low listeners.



\bibliographystyle{ACM-Reference-Format} 
\bibliography{main}

\appendix

\newpage

\section{Properties of the ARMMAN Pickup/Engagement Matrix}
\label{sec:properties}
The largest eigenvalue of the original matrix is $1.95$ times larger than the second largest value. This corresponds to the large number of zero/low listenership beneficiaries. The smallest of the eigenvalues never goes below $0.3$ times the second largest eigenvalue, showing that the matrix does not satisfy exact low-rank assumption.

\section{Hyperparameter Selection}

As mentioned in the main paper, the appropriate value of rank $C$ needs to be experimentally determined. Increasing the rank $C$ improves performance but the returns diminish beyond a point. Also increasing $C$ significantly incurs computational costs. \cref{fig:regret_cluster_simulated} shows regrets for a single run on different ranks. The final $C$ used was $5$.

The batch size $n$ used for most experiments is $1000$ and the learning rate $\epsilon$ is selected to be between $0.01 - 0.03$. This parameter depends on the value of $N$ (total samples) used. In our experiments we usually draw $1000$ samples at a time up to $35$ time steps. The learning rate needs to be scaled to keep $\frac{N \epsilon}{n}$ constant over rounds. The provided value of $\epsilon$ is for the value of $N$ at the end of all the rounds.

For our algorithms, the value of $\lambda$'s are chosen to be $1$ and the value of $\alpha$ is set such that all entries on all but two rows of $V$ have a uniform prior on $[0,1]$. However for two rows, anticipating users with very low pickup rates, we set the parameter of the exponential distribution ($\alpha$) so close to zero for the entry wise prior to be concentrated around $0$.

\begin{figure}
    \begin{subfigure}{.47\textwidth}
        \begin{center}
        \centerline{\includegraphics[width=\columnwidth]{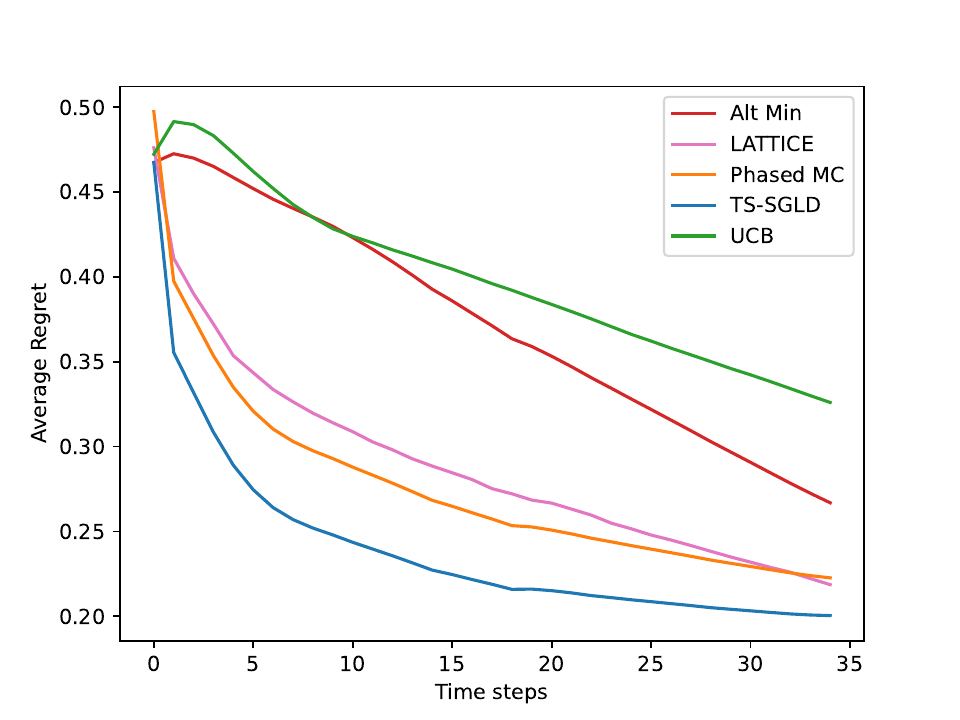}}
        \caption{Average regret for the $C$ cluster case on Simulated Data. The average regret for the different methods averaged over $15$ random matrices. Results are on a $1000 \times 20$ matrix with $5$ clusters. Every time step adds $1000$ samples.}
        \label{fig:regret_cluster_simulated}
        \end{center}
    \end{subfigure}
\end{figure}

\begin{figure*}
    \begin{center}
    \centerline{\includegraphics[width= \columnwidth]{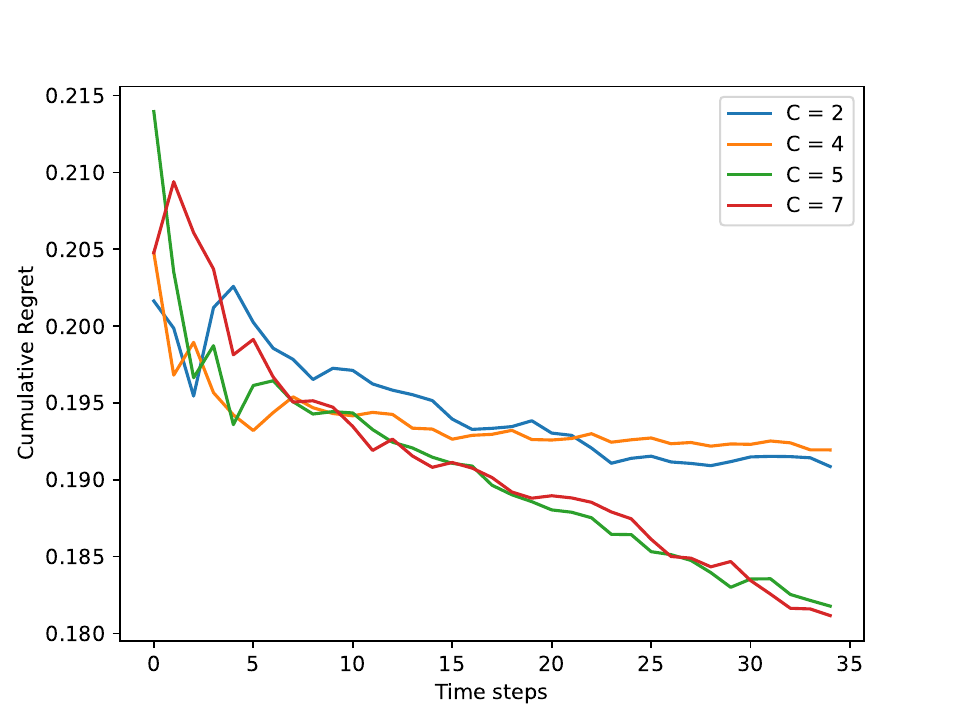}}
    \caption{Average regret for a single run for different ranks on the real world dataset. Increasing the rank reduces regret but only marginally beyond $C = 5$.}
    \label{fig:regret_rank}
    \end{center}
\end{figure*}

\begin{figure*}
    \begin{subfigure}{.45\textwidth}
        \begin{center}
        \centerline{\includegraphics[width=\columnwidth]{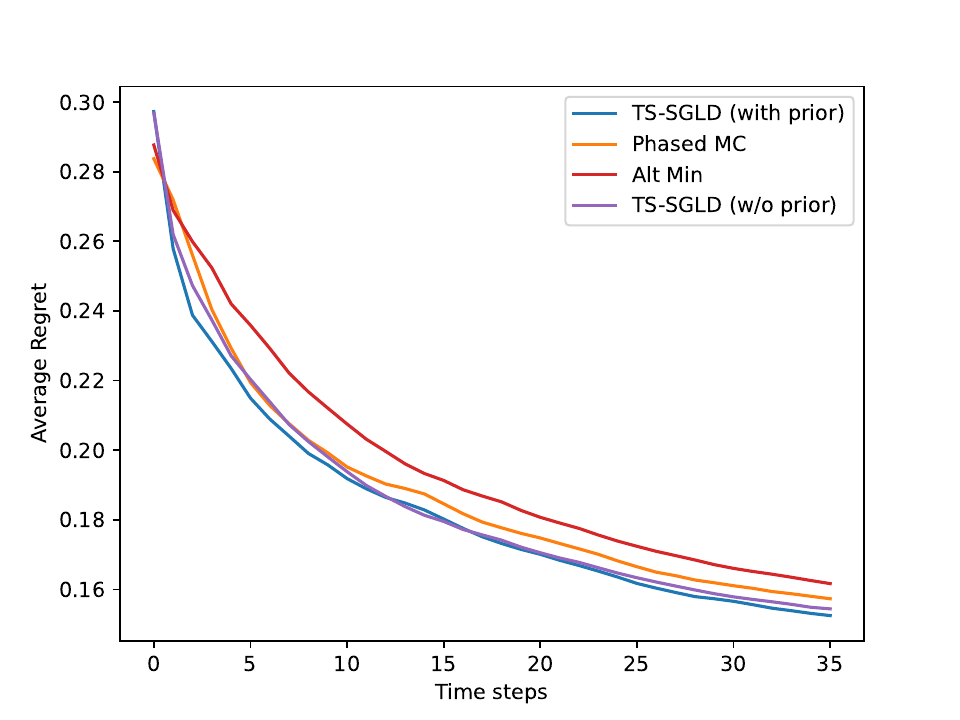}}
        \caption{Average regret on ARMMAN Data for the different methods. Results are for  $1000$ beneficiaries where every time step adds $1000$ samples. }
        \label{fig:regret_1}
        \end{center}
    \end{subfigure}
    \hspace{3mm}
    \begin{subfigure}{.45\textwidth}
        \begin{center}
        \centerline{\includegraphics[width=\columnwidth]{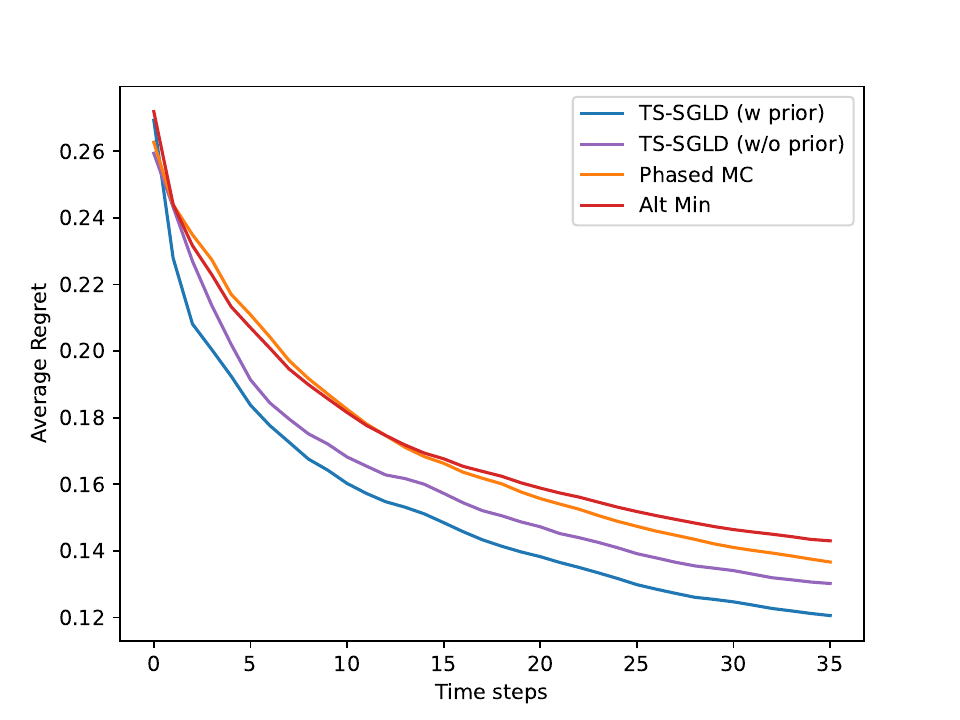}}
        \caption{Average regret for the medium pick-up rate group ($488$/$1000$ users), where every time step adds $1000$ samples for all users.}
        \label{fig:regret_middle_cum}
        \end{center}
    \end{subfigure}
    \hspace{3mm}
    \begin{subfigure}{.45\textwidth}
        \begin{center}
        \centerline{\includegraphics[width=\columnwidth]{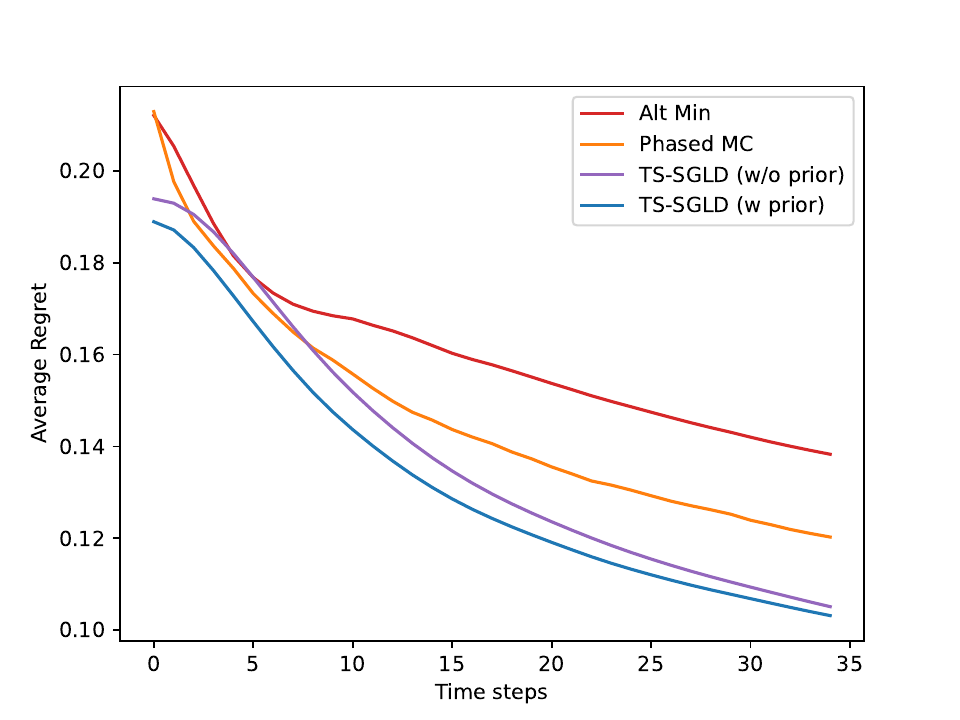}}
        \caption{Average Regret for the pickup + engagement on medium pickup group.}
        \label{fig:regret_pe_avg}
        \end{center}
    \end{subfigure}
    \hspace{3mm}
    \begin{subfigure}{.47\textwidth}
        \begin{center}
        \centerline{\includegraphics[width=\columnwidth]{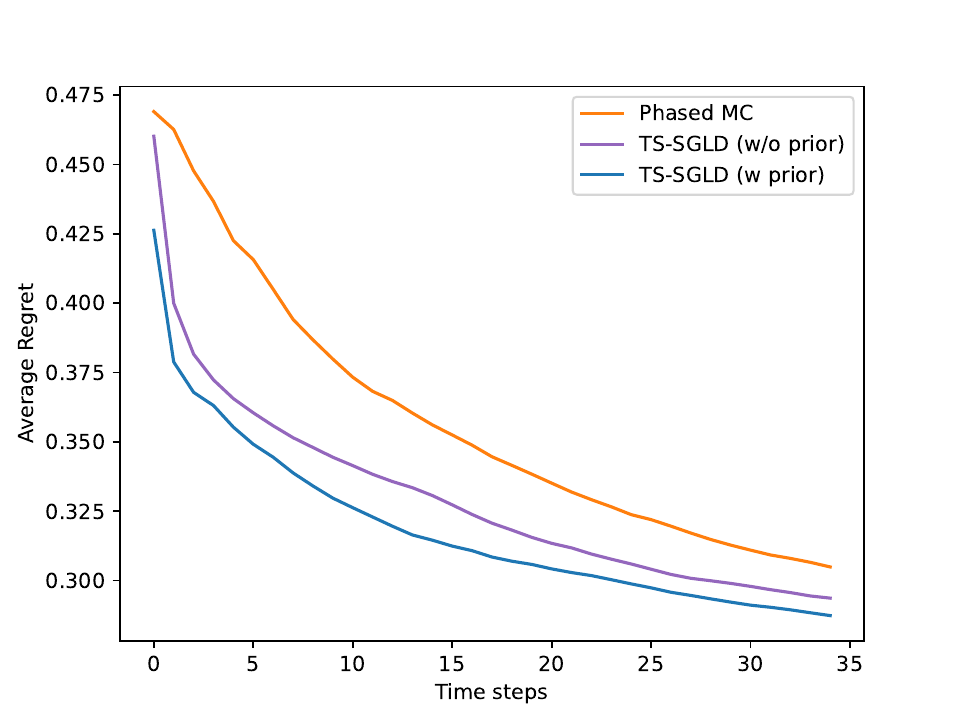}}
        \caption{Average Regret for the weekend + weekday slots on medium pickup group.}
        \label{fig:regret_wewe}
        \end{center}
    \end{subfigure}
    \hspace{3mm}
    \begin{subfigure}{.47\textwidth}
        \centerline{\includegraphics[scale=0.5]{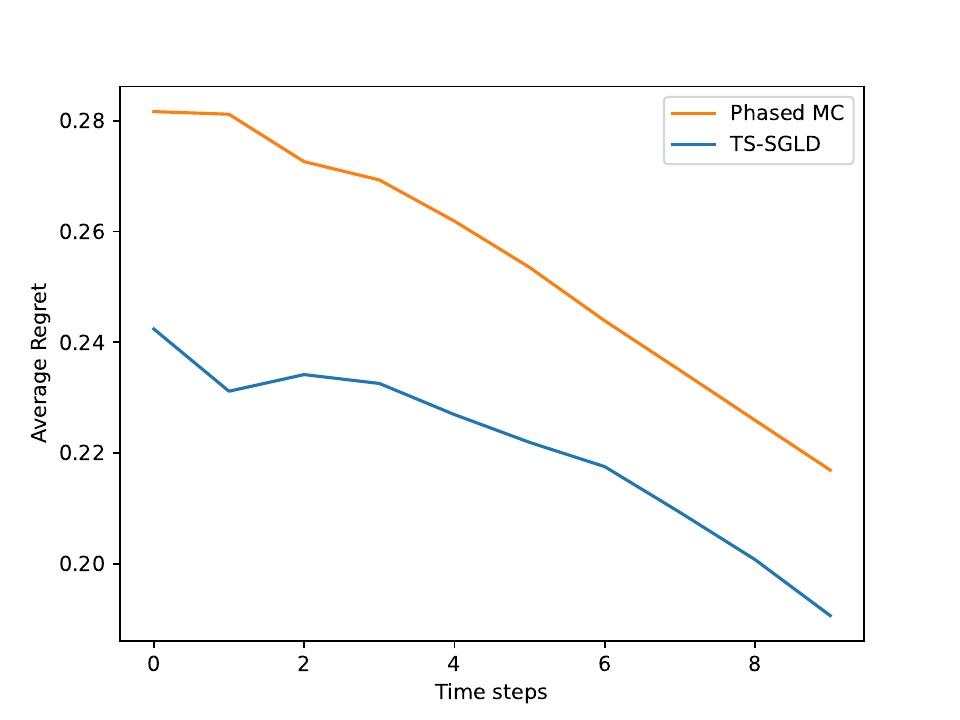}}
        \caption{Average regret on ARMMAN Data for new users added to the system. Results are for $100$ beneficiaries added to $1000$ beneficiaries where every time step adds $1000$ random samples.}
        \label{fig:regret_new}
    \end{subfigure}
    \caption{}
\end{figure*}

\section{Proof of Theoretical Results}

\begin{theorem}[Restated] In the cluster model, for $C$ user-clusters with a total of $N$ users and finite $D$ arms, the $\epsilon$-eluder dimension is at most $(2D+N)C$.
\end{theorem}
\begin{proof}
The reward function $r(u,a) = r(c(u),a)$ depends on the cluster assignment of the user $u$ and also the reward matrix as a function of the cluster id and the action index.

Denote a sequence of actions as: $\mathcal{A}^{k-1} = (u_1,a_1), \ldots u_{k-1},a_{k-1}$. Let the rewards obtained be $\mathbf{r}^{k-1} = r_1, r_2 \ldots r_{k-1} $ under a valid tuple of assignment function and Reward matrix $(c(u),R)$.

Let $\mathcal{F}^{\epsilon}(\mathcal{A}_{k-1}, \mathbf{r}^{k-1})$ denote an equivalence class of pairs $(c'(\cdot), R')$ of assignment function $c':[1:N] \rightarrow [1:C]$ and cluster-action reward matrix $R' \in \mathbb{R}^{C \times  D}$  such that the their rewards are within an $\epsilon$ radius $\ell_2$ ball of the reward sequence $\mathbf{r}^{k-1}$, i.e. $ (c',R') \in \mathcal{F}(\mathcal{A}_{k-1}, \mathbf{r}^{k-1})$ is such that
$\sqrt{\sum \limits_{i=1}^{k-1} (r_i - R'(c'(u_i),a_i))^2} < \epsilon$. We will denote the equivalence class by $\mathcal{F}^{\epsilon}_{k-1}$ for simplicity wherever suitable.


Suppose a sequence $\mathcal{A}^{K}$ is a maximal $\epsilon$-Eluder sequence of length $K$. Then, for some $\mathbf{r}^{K}$ (reward realization), the sequence of equivalence classes it generates $\{\mathcal{F}^{\epsilon}_{k}\}_{k=1}^K$ is such that P1) $\mathcal{F}^{\epsilon}_{k} \subset \mathcal{F}^{\epsilon}_{k-1}$  (strict inclusion) and $(c,R) \in \mathcal{F}^{\epsilon}_{k-1} - \mathcal{F}^{\epsilon}_{k} \Rightarrow |R(c(u_k),a_k) - r_k| > \epsilon $ P2) $|\mathcal{F}_K| = 1$ and it is the first singleton equivalence class in the sequence.

We want to argue that $K \leq (2D+N)C$. Consider an $\epsilon$-Eluder sequence $\mathcal{A}^K$ and the reward vector $\mathbf{r}^{K}$ that specifies the sequence of equivalence class. Note that the action sequence has distinct elements since if an action is repeated then the last action played would have the same reward as the one with the prior copy and property $P_1$ would be violated.

Suppose it creates a sequence of $\epsilon$-Eluder equivalence classes $\mathcal{F}^{\epsilon}_k$. Consider $(c^{*},R^{*}), (c_2,R_2) \in \mathcal{F}^{\epsilon}_{k-1}$ and $(c^{*},R^{*}) \in \mathcal{F}^{\epsilon}_{K} $ and $(c_2,R_2) \notin \mathcal{F}^{\epsilon}_k$. This has to hold for some $c_2,R_2$ since it satisfies strict inclusion property (property P1). 

This shows that $| R_2(c_2(u_k),a_k) - R^{*}(c^{*}(u_k),a_k) | > \epsilon$ but 
\[\sqrt{\sum \limits_{i=1}^{k-1} (R_2(c_2(u_k),a_k) - R^{*}(c^{*}(u_k),a_k))^2} < \epsilon\].

This implies either of the two cases:

a) $c^{*}(u_k) = c_2(u_k)=c$, implying $||R^{*}(c,a_k) - R_2(c,a_k)|| > \epsilon$ or

b) $c^{*}(u_k) \neq c_2(u_k)$.

When a 'split' happens for $\mathcal{F}^{\epsilon}_{k-1}$ we will charge it to case (a) or case (b) and assign to case (a) if both happens. Since the action sequence is distinct, let us count the number of times case (a) happens. We can now define the Eluder action sequence, in terms of cluster id and the action, is given by: $(c_{j_1},a_{k_1}) \ldots (c_{j_{\ell_1}}, a_{k_{\ell_1}})$. For the discrete action setting with a total of $M$ number of actions, the reward matrix dimension is $C \times D$. Thus, after $CD$ instances of case (a), the reward values will be identical to a previous occurrence and hence $R^{*}$ cannot show a separation beyond $\epsilon$ from any other $R_2$ since they stayed within an $\epsilon$ bound for all the prior actions (Section D.1 in \cite{russo2013eluder}). So the number of case (a) occurrences, i.e. $\ell_1$ is at most $CD$. 

Consider the case (b) occurrences given by $(u_{j_1},a_{k_1}) \ldots (u_{j_{\ell_2}},a_{k_{\ell_2}})$. Recall that we get the information that $c^{*}(u_{j_s})$ does not take some value $c'$ and $u_{j_{s}}=u'$. Suppose this is not new information, then there was the first occurrence $j_s'<j_s$ that fell into case (b) where it was first known that $c^{*}(u')$ was known to not take $c'$ and $u_{j_{s'}}=u'$. Then, it must be that case that $||R_2(c',a_{j_{s'}}) - R^{*}(c^{*}(u'),a_{j_{s'}})|| \leq \epsilon$, $||R_2(c',a_{j_{s'}}) - R_3(c',a_{j_{s'}})|| > \epsilon$, $||R^{*}(c^{*}(u'),a_{j_{s'}}) - R_3(c',a_{j_{s'}})|| > \epsilon$ and $||R_2(c',a_{j_{s}}) - 
R^{*}(c^{*}(u'),a_{j_{s}})|| > \epsilon$. This must be true for some $R_3$ that separated from $R_2,R^{*}$ at time $j_{s'}$. So every case (b) either 

b1) provides \textit{new information} about what $c^{*}$ does not take for some  user, or

b2) has an occurrence of case (a) for a pair in the past where two reward matrix disagree for the first time beyond $\epsilon$ on some $(c,a)$ (these matrices are $R_2$ and $R_3$ and for $(c',a_{j_{s'}})$) and one is in the equivalence class containing $R^{*}$.

Case (b1) can occur only at most $NC$ times since for every distinct user one has $C-1$ incorrect values in the range of $c^{*}$ to eliminate. Case (b2) can occur only at most $MC$ times (from the analysis of case (a) since both cases are the same).

Therefore, total number of case $(b)$ occurrences, i.e $\ell_2$ is at most $CD + NC$. 

Therefore, $K = \ell_1 + \ell_2 = 2CD + NC$.

\end{proof}

\begin{theorem}[Restated] For the cluster model, with $C$ user-clusters with a total of $N$ users and $D$-dimensional infinite arm set $\mathcal{A}$, the $\epsilon$-eluder dimension is $\mathcal{O}((2D+N)C)$.
\end{theorem}
\begin{proof}

The ideas from Theorem \ref{thm:condind} can be extended for the infinite arm case, with action embeddings having a $D$-dimensional representation.

Let $\Theta \in \mathcal{M} \subset \mathbb{R}^{C \times D}$ denote the cluster-reward mapping such that $R(u,a) = e_{c(u)}^T \Theta a = \Theta(c(u),:)^T a$, where $e_{c(u)}$ denotes an indicator function denoting the cluster of user $u$. We additionally assume bounded $L_2$-norms, meaning there exist constants $S$ and $\gamma$ such that for all $a \in \mathcal{A}$ and $\Theta \in \mathcal{M} $, $||\Theta_{c,:}||_2 \leq S$ and $||a||_2 < \gamma$.

Now, in the sequence of Eluder equivalence classes $\mathcal{F}^k$, we again consider the separation condition as in Theorem \ref{thm:condind}, leading to equivalent two cases:

a) $c^*(u_k) = c_2(u_k) = c$, implying $||R^*(c, a_k) - R_2(c, a_k)|| > \epsilon$ and
b) $c^*(u_k) \neq c_2(u_k)$

Consider the sub-sequence which is $\epsilon$-Eluder and satisfies case (a): $(c_{j_1},a_{k_1}) \ldots (c_{j_{\ell_1}},a_{k_{\ell_1}})  $. Consider only the sub-sequence in this where $c=c_{j_i}$ and $j_i,k_i$ be the last index. This is also $\epsilon$ Eluder.

We have $||R^*(c,a_{k_{i}}) - R_2(c,a_{k_{i}})|| = ||(\Theta^*(c,:) - \Theta_2(c,:))^T a_{k_{i}}|| > \epsilon$ and $L_2$ norm of all other rewards before in the sub-sequence is less than $\epsilon$. Since only $c$-th row of $\Theta, \Theta_2$ is involved, this is effectively a linear model.

Under the norm conditions ($S$ and $\gamma$ bounds), we invoke Proposition 6 from \cite{russo2013eluder} for the linear model, which leads to a $\mathcal{O}(D \log(1+ 2S/\epsilon^2))$ bound on the Eluder sequence length of this sub-sequence. 

There are $C$ different sub-sequences at most. Therefore, $\ell_1 \leq \mathcal{O}(CD \log(1+2S/\epsilon^2)) $

For case (b), the subcase (b1) remains unchanged ($CM$), while (b2) is again limited by $\mathcal{O}(CD \log (1+2S/\epsilon^2))$ by an analogous argument.

Thus, $K = \ell_1 + \ell_2 = \mathcal{O}(2DC\log (1+2S/\epsilon^2) + NC)$.
\end{proof}


\end{document}